\documentclass[letterpaper]{article} 
\usepackage{aaai23}  
\usepackage{times}  
\usepackage{helvet}  
\usepackage{courier}  
\usepackage[hyphens]{url}  
\usepackage{graphicx} 
\usepackage{natbib}  
\usepackage{caption} 
\usepackage{algorithm}
\usepackage{algorithmic}
\usepackage{listings}
\usepackage{amsmath}
\usepackage{amsfonts}
\usepackage{xcolor}
\usepackage{tikz}
\usetikzlibrary{bayesnet}

\definecolor{mygreen}{rgb}{0.09,0.5,0.14}

\definecolor{commentcolor}{RGB}{110,154,155}
\definecolor{inputcolor}{RGB}{255, 105, 180}   
\newcommand{\PyComment}[1]{\ttfamily\textcolor{commentcolor}{\# #1}}  
\newcommand{\PyInput}[1]{\ttfamily\textcolor{inputcolor}{\# #1}}
\newcommand{\PyCode}[1]{\ttfamily\textcolor{black}{#1}} 
\newcommand{\mysection}[1]{\vspace{-4pt}\section{#1}}
\newcommand{\mysubsection}[1]{\vspace{-4pt}\subsection{#1}}
\newcommand{\mysubsubsection}[1]{\vspace{-2pt}\subsubsection{#1}}
\newcommand{\myparagraph}[1]{\vspace{-2pt}\paragraph{#1}}

\DeclareMathOperator*{\argmax}{arg\,max}

\usepackage{newfloat}
\usepackage{listings}
\DeclareCaptionStyle{ruled}{labelfont=normalfont,labelsep=colon,strut=off} 
\floatstyle{ruled}
\newfloat{listing}{tb}{lst}{}
\floatname{listing}{Listing}
\usepackage{qiangstyle}
\title{Metric Residual Networks for Sample Efficient \\
Goal-Conditioned Reinforcement Learning}
\author{
    Bo Liu\textsuperscript{\rm 1}, Yihao Feng\textsuperscript{\rm 1}, Qiang Liu\textsuperscript{\rm 1}, Peter Stone\textsuperscript{\rm 1,2}
}
\affiliations{
    \{bliu, yihao, lqiang, pstone\}@cs.utexas.edu\\
    \textsuperscript{\rm 1}The University of Texas at Austin, \textsuperscript{\rm 2}Sony AI\\
}

\begin{document}

\maketitle

\begin{abstract}
Goal-conditioned reinforcement learning (GCRL) has a wide range of potential real-world applications, including manipulation and navigation problems in robotics. Especially in such robotics tasks, sample efficiency is of the utmost importance for GCRL since, by default, the agent is only rewarded when it reaches its goal. While several methods have been proposed to improve the sample efficiency of GCRL, one relatively under-studied approach is the design of neural architectures to support sample efficiency.
In this work, we introduce a novel neural architecture for GCRL that achieves significantly better sample efficiency than the commonly-used monolithic network architecture. 
The key insight is that the 
\emph{optimal} action-value function $Q^*(s, a, g)$ must satisfy the triangle inequality in a specific sense.
Furthermore, we introduce the \emph{metric residual network} (MRN)  that deliberately decomposes the action-value function $Q(s,a,g)$ into the negated summation of a metric plus a residual asymmetric component. MRN provably approximates any optimal action-value function $Q^*(s,a,g)$, thus making it a fitting neural architecture for GCRL.
We conduct comprehensive experiments across 12 standard benchmark environments in GCRL. The empirical results demonstrate that MRN uniformly outperforms other state-of-the-art GCRL neural architectures in terms of sample efficiency. The code is publicly available at \textcolor{magenta}{\url{https://github.com/Cranial-XIX/metric-residual-network}}.
\end{abstract}
\vspace{-5pt}

\mysection{Introduction}
Goal-conditioned reinforcement learning (GCRL) refers to the problem in which an agent learns to solve a set of tasks indicated by different ``goals" via trial and error. 
In contrast to the standard reinforcement learning (RL) setting, in which the per-step reward the agent receives can be an arbitrary fixed scalar function, the reward function in GCRL is usually an indicator function identifying whether the agent has achieved the goal but with a varying goal. 
As a result, GCRL enables the learning of a whole family of tasks with relatively little human effort toward reward design and thus has many potential real-world applications.
For instance, robot manipulation tasks like picking and placing an object in a target location can be viewed as a GCRL problem where the underlying task family is parameterized by the target goal location. 
Similarly, robot navigation can be viewed as a GCRL problem where the goal can be any navigation destination.

Although GCRL has a straightforward reward formulation, it inherits the challenges associated with sparse reward learning: unlike in the \emph{dense} reward setting, the reward signal is only informative when the agent reaches the goal. 
Therefore, sample efficiency is a major challenge in GCRL, meaning that the agent typically needs a large number of interactions with the environment to make meaningful learning progress. 
To address this problem, an active line of research focuses on designing novel learning algorithms that efficiently use the data. One popular approach is hindsight experience replay (HER)~\citep{andrychowicz2017hindsight}, which relabels the agent's trajectories as if they were aiming to reach the state that they in fact reached, thus rendering every trajectory an example of a successful goal achievement. 

One relatively under-explored direction is the design of better neural architectures for  GCRL. 
For actor-critic-like methods, prior work has proposed decomposing the critic function (a.k.a the action-value function $Q(s,a,g)$, see Sec.~\ref{sec:gcrl}) into a bilinear network, e.g., either $Q(s, a, g) = f(s,a)^\top \phi(g)$~\citep{schaul2015universal} or $Q(s, a, g) = f(s, a)^\top \phi(s, g)$~\citep{hong2022bilinear}, where $f$ and $g$ are separate neural modules. The principle behind these designs is to inject useful inductive bias into the architecture.  Although empirically found to be effective, it remains unclear why such a decomposition works and whether it can be improved upon.

In this work, we argue that one fundamental inductive bias is that under the sparse reward setting, the negated optimal action-value (e.g., $-Q^*(s, a, g)$) must satisfy the \emph{triangle inequality} in a specific sense.
While there exists prior work on neural architectures that respect the triangle inequality~\citep{pitis2020inductive}, we make the following novel contributions:
\begin{itemize}
    \item We are the first to show that the discounted $Q^*(s, a, g)$ in the standard GCRL setting (See Sec.~\ref{sec:gcrl}), when the goal is a deterministic onto mapping from the state, satisfies the triangle inequality.
    \item Motivated by the first point, we introduce the metric residual network (MRN) that deliberately decomposes the negative action-value function (e.g., $-Q$) into the sum of a metric and an asymmetric residual component, which provably approximates \emph{any} quasipseudometric.\footnote{$(d, \gX)$ defines a quasipseudometric on $\gX$ if 1) $\forall x \in \gX,~~d(x,x)=0$ and 2) $\forall x,y,z \in \gX,~~d(x,y)+d(y,z)\geq d(x,z)$.}
    \item With a comprehensive experiment on 12 standard GCRL benchmark environments, MRN consistently performs better than a range of common prior designs. We hypothesize that MRN's metric component speeds up learning, and the asymmetric residual component makes the approximation accurate.
\end{itemize}

\begin{figure*}[t]
    \centering
    \includegraphics[width=\textwidth]{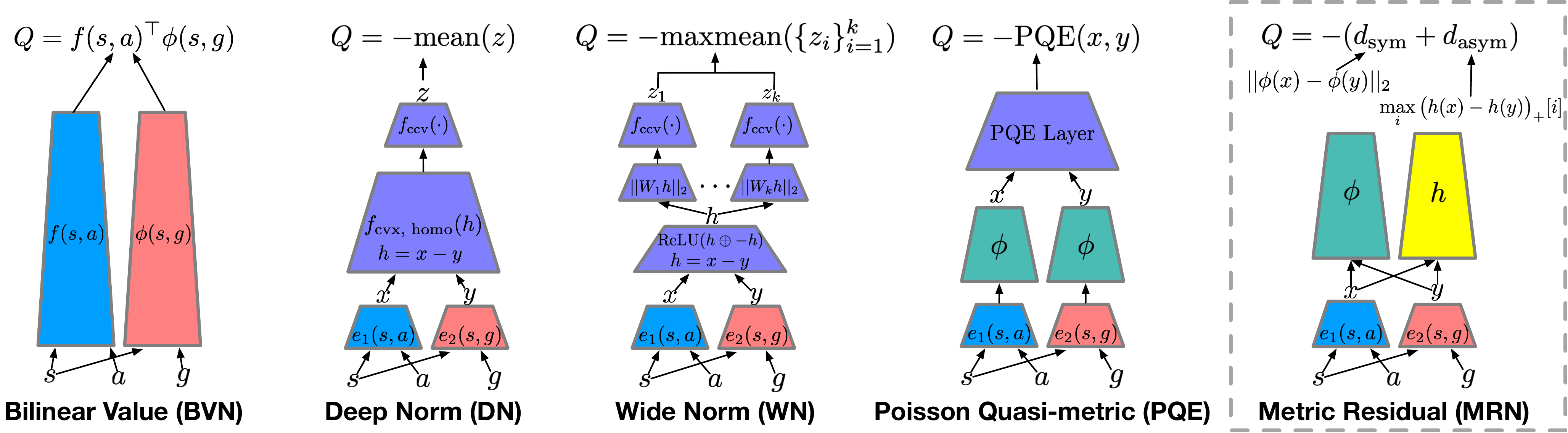}
    \caption{Comparison of different neural architecture designs for decomposing the action value function $Q(s,a,g)$ in GCRL. For the metric residual network (MRN), $d_\text{sym}(x,y) = ||\phi(x) - \phi(y)||_2$ and $d_\text{asym}(x,y) = \max_i \big(h(x) - h(y)\big)_+[i]$. Note that $x$ and $y$ are passed through the \emph{same} networks $\phi$ and $h$. In experiments, all networks are created with approximately the same number of parameters. In particular, the size of $\phi$ + $e_1$ (or $h$ + $e_2$) in MRN is the same as that of $f$ (or $\phi$) in BVN.}
    \label{fig:architecture}
    \vspace{-10pt}
\end{figure*}
\mysection{Preliminaries}
In this section, we first introduce the formal definition of the GCRL problem. 
Then we review the off-policy actor-critic algorithm DDPG and the hindsight experience replay (HER) method with goal relabeling for GCRL. DDPG+HER serves as the substrate algorithm within which we compare different neural architecture designs.

\mysubsection{Goal-conditioned Reinforcement Learning}
\label{sec:gcrl}
The goal-conditioned reinforcement learning problem can be formulated as a goal-conditioned Markov Decision Process, which is an 8-tuple $M_\text{gc} = (\gS, \gA, \gG, T, R, \gamma, \rho_0, \rho_\gG)$. Here, $\gS$, $\gA$, and $\gG$ are the state, action, and goal spaces of the agent. $T: \gS \times \gA \rightarrow \gS$ is the transition dynamics. $R: \gS \times \gA \times \gG \rightarrow \mathbb{R}$ is the reward function. $\gamma$ is the discount factor. $\rho_0$ and $\rho_\gG$ are the initial state and target goal distributions. Concretely, at the start of each episode, an initial state $s_0 \sim \rho_0$ and a goal $g \sim \rho_\gG$ are sampled. The agent starts at $s_0$, and at each time step $t \geq 0$, it chooses action from a policy $\pi$, e.g., $a_t \sim \pi(\cdot \mid s_t)$. It then receives a reward $r_{t,g} = R(s_t, a_t, g)$ and transitions to the next state $s_{t+1} \sim T(\cdot \mid s_t, a_t)$. The overall objective of the infinite-horizon GCRL problems is to maximize the following:
\begin{equation}
    \max_\pi J(\pi) = \mathbb{E}_{\substack{s_0 \sim \rho_0, ~g \sim \rho_\gG, \\ (s_t, a_t, r_{t,g}) \sim \pi, T, R}}\bigg[ \sum_{t=0}^\infty \gamma^t r_{t,g} ~\big|~s_0, g\bigg],
\end{equation}
where the reward $r_{t, g}$ is often defined as\footnote{There are two conventional setups: either the agent receives a reward of $1$ when reaching the goal and $0$ otherwise, or it receives $-1$ reward until it reaches the goal. In this work, without further specification, we assume the latter setting.}
\begin{equation}
    r_{t, g} = R(s_t, a_t, g) = \begin{cases}
    0 & M(s_t, a_t) = g \\
    -1 & \text{otherwise} \\
    \end{cases}.
    \label{eq:reward}
\end{equation}
Here $M: \gS \times \gA \rightarrow \gG$ is an onto mapping from the product space of the state and action to the goal space. Note that $M_\text{gc}$ does \emph{not} terminate when the goal is reached; the agent receives a reward every time it returns to the goal (or remains there).
In practice, usually $\gG \subset \gS$. For example, if the goal is for an autonomous car to reach any given velocity, then $\gG$ would naturally comprise the velocity component of the agent's full state space $\gS$. For technical reasons in Sec.~\ref{sec:method-triangle-ineq}, we also include $a_t$ in deciding whether the goal is reached.\footnote{When $r_{t,g}$ is determined by checking whether $M(s_t, a_t) = g$, as opposed to whether $M(s_{t+1} = g)$, it is a deterministic function involving no randomness from the transition dynamics $T$.}

\mysubsection{Off-Policy Actor-Critic}
Actor-critic methods like Deep Deterministic Policy Gradient (DDPG)~\citep{lillicrap2015continuous}, Twin Delayed DDPG (TD3)~\citep{fujimoto2018addressing} and Soft Actor-Critic (SAC)~\citep{haarnoja2018soft} are popular methods for solving GCRL. In this work, we consider DDPG as the underlying RL algorithm, as it is the most commonly used method in prior GCRL research. In DDPG and under the GCRL setting, the critic $Q(s, a, g)$ evaluates the expected discounted return to reach goal $g$ at state $s$ by choosing $a$. In particular, $Q(s, a, g)$ is called the \emph{universal value function approximator} (UVFA) as it extends the normal $Q(s, a)$ to a subset of goals indicated by $g$. Formally, we have
\begin{equation}
    \resizebox{0.9\linewidth}{!}{%
    $Q^\pi(s, a, g) = \mathbb{E}_{\substack{(s_t, a_t, r_{t,g}) \\\sim \pi, T, R}} \bigg[ \sum\limits_{t=0}^\infty \gamma^t r_{t,g} ~\big|~s_0=s, a_0=a, g\bigg].
$%
}
\end{equation}

The critic $Q$ is then updated by minimizing the temporal-difference (TD) error:
\begin{equation}
    \resizebox{0.9\linewidth}{!}{%
    $L(Q) = \mathbb{E}\bigg[ \bigg(r_{t,g} + \gamma Q(s_{t+1}, \pi(s_{t+1}), g) - Q(s_t, a_t, g) \bigg)^2\bigg]\,,
$%
}
\end{equation}
where the expectation is taken over $(s_t, a_t, s_{t+1}, g) \sim \gD$, and $\gD$ is the replay buffer that stores the agent's previous experience. The actor (\textit{a.k.a} the policy $\pi$) is then updated through the critic. In particular, the policy's gradient in DDPG is
\begin{equation}
    \nabla_\pi J(\pi) = \mathbb{E}\bigg[ \nabla_{a_t} Q(s_t, a_t, g)~\big|_{a_t = \pi(s_t)}\bigg].
\end{equation}

\mysubsection{Hindsight Experience Replay}
A popular technique for mitigating the sparse reward problem is hindsight experience replay (HER)~\citep{andrychowicz2017hindsight}. Assume the agent collects a trajectory $\tau = (s_0, a_0, \dots, s_T, a_T)$ using policy $\pi$ and corresponding initial state $s_0$ and goal $g$. If $\forall t, M(s_t, a_t) \neq g$, then the agent receives $-1$ reward every step and therefore learns nothing from this rollout trajectory. HER relabels the trajectory as if the agent were pursuing the goal $g \in \{M(s_1, a_1), \dots, M(s_T, a_T)\}$.
In this case, the trajectory can be viewed as a successful example of reaching a different goal. By representing $Q$ as a UVFA (e.g., $Q(s, a, g)$), one hopes that learning on these relabeled trajectories can help to generalize $Q$ to all different goals in the goal space, including those from the original goal distribution $\rho_\gG$. In practice, HER is often combined with an off-policy actor-critic algorithm. Following the original work of HER and prior work in GCRL, we use DDPG+HER as the base GCRL algorithm for policy learning and compare the performance across different neural architectures within the critic.
\mysection{Metric Residual Network}
One way to design a good neural architecture for UVFA is to base it on theoretically sound inductive biases so that the designed networks naturally inherit these inductive biases. Metric residual network (MRN) is designed based on the observation that the optimal universal action-value function $Q^*(s, a, g)$ in GCRL must satisfy the triangle inequality in a specific way. We formally prove this observation is correct in Sec.~\ref{sec:method-triangle-ineq} and introduce a novel network architecture that enforces the triangle inequality in Sec.~\ref{sec:method-mrn}.

\mysubsection{Triangle Inequality in GCRL}
\label{sec:method-triangle-ineq}
In this section, we start by showing that when $\gG \equiv \gS \times \gA$,~$Q^*(s, a, g)$ satisfies the triangle inequality. Then we extend the result to the general case when $\gG \not\equiv \gS \times \gA$.
\mysubsubsection{When $\bm{\gG \equiv \gS \times \gA}$}
Under this setting, $M$ (in Eq.~\eqref{eq:reward}) becomes the identity mapping. For convenience of notation, let $\gX = \gS \times \gA$, ~i.e. $x_t = (s_t, a_t)$. Given a policy $\pi$, the universal value function on $\gX$ then becomes
\begin{equation}
    Q^\pi(x, x_g) = \mathbb{E}_{\substack{(x_t, r_t) \\\sim \pi, T, R}} \bigg[ \sum_{t=0}^\infty \gamma^t r_{t,g} ~\big|~x_0 = x, g = x_g\bigg].
\label{eq:specific-v-pi}
\end{equation}
The optimal universal value function is therefore:
\begin{equation}
Q^*(x, x_g) = \max_\pi Q^\pi(x, x_g).
\label{eq:specific-v}
\end{equation}
To build intuition regarding why $Q^*$ satisfies the triangle inequality, consider any $x^1, x^2$ and $x^3$ in $\gX$ (Fig.~\ref{fig:intuition}).

\begin{figure}[h!]
    \centering
    \includegraphics[width=0.8\columnwidth]{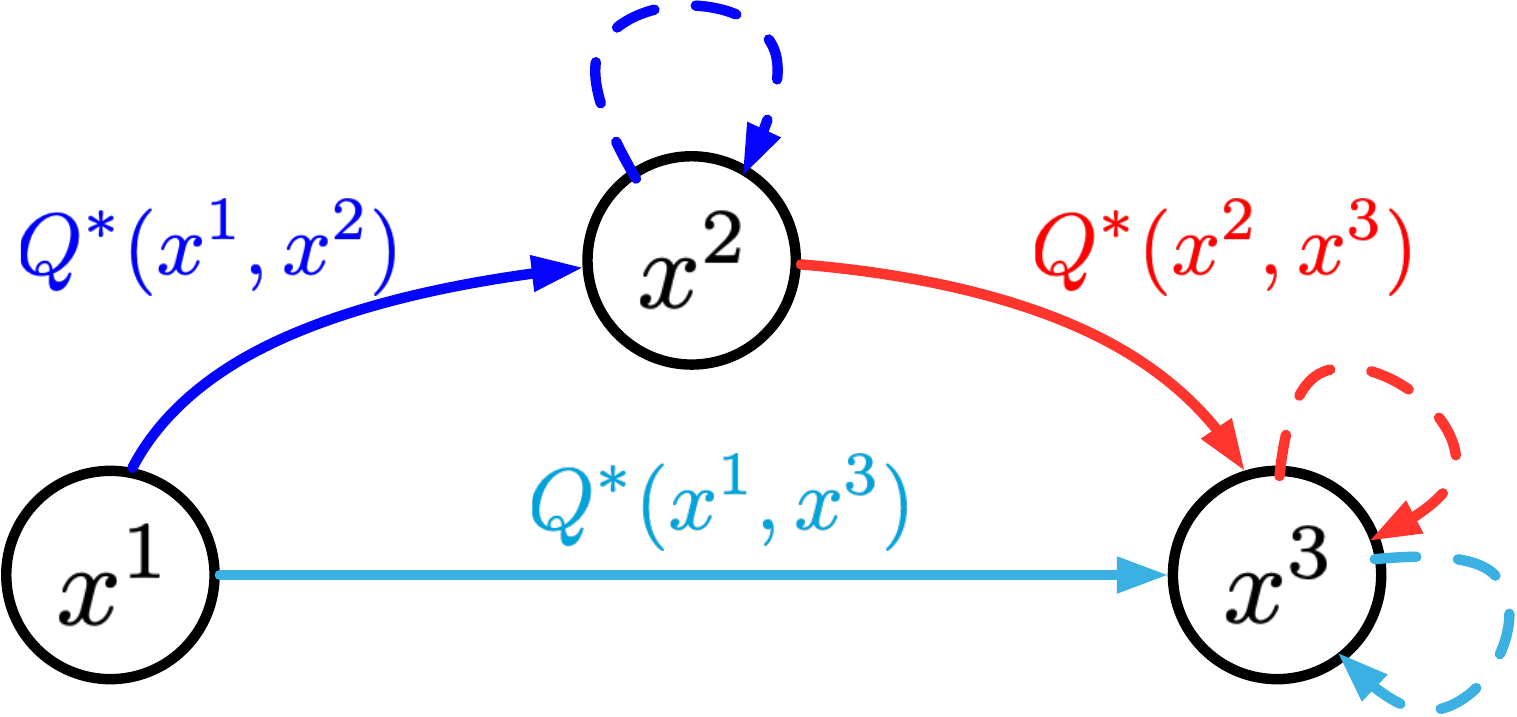}
    \vspace{-8pt}
    \caption{Intuition for triangle inequality in GCRL.}
    \label{fig:intuition}
\end{figure}
Think of $Q^*(x^1, x^2)$ as consisting of the ``cost'' of (discounted reward for) reaching $x^2$ ($\color{blue}{\rightarrow}$) plus the ``cost'' of staying at $x^2$ ($\color{blue}{\dashrightarrow}$). Decompose $Q^*(x^2, x^3)$ and $Q^*(x^1, x^3)$ similarly. Then, by definition of $Q^*(x^1, x^3)$, a proof sketch is
\begin{equation}
    \begin{split}
    \color{cyan}{Q^*(x^1, x^3)} &= ({\color{cyan}{\rightarrow}}) + ({\color{cyan}{\dashrightarrow}}) \geq ({\color{blue}{\rightarrow}}) + ({\color{red}{\rightarrow}}) + ({\color{red}{\dashrightarrow}}) \\
    &\geq \color{blue}{Q^*(x^1, x^2)} + \color{red}{Q^*(x^2, x^3)}.
    \end{split}
\end{equation}
This statement is formally presented in Thm.~\ref{thm:triangle-inq}, and the proof is provided in Appendix~\ref{sec:apx-thm1}.

\begin{thm} Consider the following Goal-conditioned MDP $M_{gc} = (\gS, \gA, \gG, T, R, \gamma, \rho_0, \rho_g)$ as described in Sec.~\ref{sec:gcrl} except that $\gG\equiv \gS \times \gA$ . Then the optimal universal value function $Q^*$ defined in Eq.~\eqref{eq:specific-v} satisfies the triangle inequality: $~\forall~ x^1, x^2, x^3 \in \gX$,
    \begin{equation}
       Q^*(x^1, x^2) + Q^*(x^2, x^3) \leq Q^*(x^1, x^3).
    \end{equation}
\label{thm:triangle-inq}
\vspace{-10pt}
\end{thm}

Concurrently, \citet{wang2022learning} present a claim similar to Thm.~\ref{thm:triangle-inq}, where they show the optimal goal-reaching plan costs in MDPs form a quasipseudometric. The optimal goal-reaching plan cost $C(x_1, x_2)$ can be viewed as the minimum expected first hitting time from $x_1$ to $x_2$. We emphasize the two differences between Thm.~\ref{thm:triangle-inq} and their statement: \textbf{1)} the formulation in Sec.~\ref{sec:gcrl} does not assume the agent \emph{terminates} once reaching the goal.  Thus the optimal behavior will be \emph{staying} at the goal as long as possible, if necessary, leaving the goal state and returning as quickly as possible. By contrast, the optimal behavior for the first hitting time problem is \emph{hitting} the goal as soon as possible, even if the agent passes the goal, never to return. In practice, sometimes it is possible to change one type of problem to the other by modifying the underlying MDP, but we submit that the problem formulation we consider in Sec.~\ref{sec:gcrl} is more relevant and general, as it is the setting used in almost all common GCRL benchmarks and it can always subsume the first hitting time problem by enlarging the state space.\footnote{One can introduce an extra sink state where the agent stays once it reaches the goal to subsume the first hitting time problem into our formulation.} \textbf{2)} on the other hand, \citet{wang2022learning} state that the optimal goal-reaching plan costs form a quasipseudometric, while we only claim that $Q^*$ satisfies the triangle inequality. The subtle difference is that within the formulation in Sec.~\ref{sec:gcrl}, it is possible that $Q^*(x, x) \neq 0$ for some $x \in \gX$.\footnote{For instance, there are no actions that ensure the agent can stay at $x$ forever, i.e., it needs to come out and return to $x$ repeatedly.} Hence, our formulation is less restrictive.

However, as we will show next, one can still approximate $Q^*$ well using a quasipseudometric in a larger space. For instance, assume we have a copy of $x$ denoted $\hat{x}$ for every $x \in \gX$. Call the space $\gY = \gX \cup \hat{\gX}$, where $\hat{\gX} = \{ \hat{x} \}$. Then consider the following function $\hat{Q}: \gY \times \gY \rightarrow \mathbb{R}_{\leq 0}$:
\begin{equation}
    \hat{Q}^*(a, b) = \begin{cases}
        0 & a = b \\
        Q^*(a, a) & b = \hat{a}, ~a \in \gX \\
        Q^*(a, b) & a\neq b, ~a,b \in \gX \\
        -\infty & \text{otherwise}.
    \end{cases}
    \label{eq:transform}
\end{equation}
It is easy to check that \textbf{1)} $-\hat{Q}^*$ is a quasipseudometric on $\gY$, \textbf{2)} $\forall x, y \in \gX$, $\hat{Q}^*(x, e(y)) = Q^*(x, y)$, where $e(y) = y$ if $x \neq y$ and $e(y) = \hat{y}$ if $x = y$. Therefore, though $Q^*$ only needs to satisfy triangle inequality on $\gX$, by mapping $y$ to $e(y)$, $\hat{Q}^*$, a negated quasipseudometric, represents $Q^*$.
\mysubsubsection{When $\bm{\gG \not\equiv \gS \times \gA}$} Under this setting, $M$ (in Eq.~\eqref{eq:reward}) is an onto mapping, meaning that there might exist multiple $(s, a)$ pairs mapping to the same goal. If the underlying transition dynamics $T$ is \emph{deterministic}, one observes that
\begin{equation}
    Q^*(x, g) = \sup_{x': M(x') = g} Q^*(x, x').
    \label{eq:qxg_is_qxx}
\end{equation}
Therefore, when the $\sup$ is attainable, denote $x_g = \argmax_{x': M(x') = g} Q^*(x, x')$ then $Q^*(x, g) = Q^*(x, x_g)$. The proof and more discussion are in Appendix~\ref{sec:apx-obs}. 
%
As a result, we can view $Q^*(s, a, g)$ as a specific ``distance"\footnote{Quotation marks are used to indicate that $Q^*$ might not be a true distance function as it is not guaranteed to be symmetric.} between $x=(s,a)$ and $x_g$, which leads to our design of the neural architecture in the following section. 

\mysubsection{Network Design}
\label{sec:method-mrn}
In this section, we first present a novel construction for quasipseudometrics and prove that it universally approximates any quasipseudometric. Then we present a novel neural architecture for GCRL based on this construction and show in a toy example why such design improves generalization and therefore sample efficiency. Lastly, we provide a unified view of existing architecture designs for GCRL.
\mysubsubsection{A Novel Construction for Quasipseudometrics} From the observations in Eq.~\eqref{eq:transform}-\eqref{eq:qxg_is_qxx}, we know $Q^*(s,a,g) = Q^*(x, x_g)$, where the latter can be represented by a negated quasipseudometric $\hat{Q}^*$. Therefore, we first present a novel construction for quasipseudometrics in Prop.~\ref{prop:quasipseudometric}.

\begin{pro}
    Assume $x, y \in \gX$, define
\begin{equation}
    d(x, y) = d_\text{sym}(x, y) + \underbrace{\max_{i \in [K]} \big(h_i(x) - h_i(y)\big)_{+}}_{d_\text{asym}(x, y)},
    \label{eq:qm}
\end{equation}
where $(\cdot)_+ = \max(\cdot, 0)$ is the Rectified Linear Unit (ReLU) function, $h_i: \gX \rightarrow \mathbb{R}$, and $(\gX, d_\text{sym})$ is a metric. Then $d$ ensures the three axioms of quasipseudometrics:
\begin{itemize}
    \item Non-negativity:
    ~~~~~~~$\forall~ x, y \in \gX, ~~~~ d(x,y) \geq 0.
    $
    \item Identity:
    ~~~~~~~$\forall~x \in \gX, ~~~~ d(x, x) = 0.
    $
    \item Triangle inequality:
    $$
        \forall~ x, y, z \in \gX, ~~~~d(x,z) \leq d(x, y) + d(y,z).
    $$
\end{itemize}
In practice, we choose $d_\text{sym}(x, y) = ||\phi(x) - \phi(y)||_2$ where $\phi$ is an arbitrary neural network. Clearly, $\forall x, y$, $d_\text{sym}(x,y) = d_\text{sym}(y,x)$, and with the ReLU function, $d_\text{asym}$ is asymmetric, i.e. $\exists x,y$, $d_\text{asym}(x,y) \neq d_\text{asym}(y,x)$.
\label{prop:quasipseudometric}
\end{pro}
The proof is provided in Appendix~\ref{sec:apx-prop-quasipseudometric}.
Prop.~\ref{prop:quasipseudometric} ensures that any construction of $d$ in the form of Eq.~\eqref{eq:qm} is a quasipseudometric, but does not guarantee that any quasipseudometric on $\gX$ can be represented in the form of $d$. The following theorem confirms that this is indeed the case.
\begin{thm}[Universal approximation of $d$]
Given any continuous function $\nu:\gX \times \gX \rightarrow \mathbb{R}$, where $(\gX, \nu)$ is a quasipseudometric and $\gX$ is compact. Then $\forall \epsilon > 0$, with a sufficiently large $K$, there exists a quasipseudometric $d$ in the form of Eq.~\eqref{eq:qm} such that
$$
\forall x, y \in \gX, ~~|d(x, y) - \nu(x, y)| \leq \epsilon.
$$
\vspace{-12pt}
\label{thm-universal}
\end{thm}
The proof is provided in Appendix~\ref{sec:apx-thm2}.

\mysubsubsection{A Novel Architecture for GCRL}
Based on the novel construction of a quasipseudometric in Eq.~\eqref{eq:qm}, we introduce the metric residual network (MRN) for GCRL problems, which consists of two parts:
\begin{itemize}
    \item \textit{Projection}: with two encoders $e_1: \gS \times \gA \rightarrow \gZ$ and $e_2: \gS \times \gG \rightarrow \gZ$, we project $(s, a, g)$ to two latent vectors:
        $$h_{sa} = e_1(s, a),~~~h_{sg} = e_2(s, g).$$
    \item \textit{Enforcing triangle inequality}: with $d_\text{sym}$ and $d_\text{asym}$ defined in Eq.~\eqref{eq:qm}, represent:
    \vspace{-3pt}
        $$Q(s,a,g) = - \bigg(d_\text{sym}(h_{sa}, h_{sg}) + d_\text{aysm}(h_{sa}, h_{sg})\bigg).$$ 
    \vspace{-10pt}
\end{itemize}
In the projection step, one can view $e_1(s, a)$ and $e_2(s, g)$ as summarizing the sufficient statistics of $x$ and $x_g$ and projecting them to the same latent space $\gZ$. In theory, $x_g$ should be a function of $s, a, g$. The reason why we omit $a$ from the inputs to $e_2$ is that otherwise $e_2$ has all the information needed to predict $Q^*$ directly, which makes the decomposition to $x$ and $x_g$ meaningless and thus slows down training (See the ablation study in Sec.~\ref{sec:exp-ablation} for empirical evidence). After the projection step, $d_\text{sym}$ and $d_\text{asym}$ act as the inductive bias that constrains the learning on the latent space $\gZ$. \textbf{The core intuition} behind decomposing $d$ into $d_\text{sym}$ and $d_\text{asym}$ is that $d_\text{sym}$, due to its symmetry, \emph{improves sample efficiency}, while $d_\text{asym}$ \emph{ensures the approximation is accurate}.
\mysubsubsection{Toy Example} We provide a toy example in Fig.~\ref{fig:toy} (top-left) to validate \textbf{1)} $d_\text{sym}$ helps improve generalization, and \textbf{2)} $d_\text{asym}$ helps improve the modeling accuracy. The environment is a square of size $1$. The agent navigates freely in the white region (w/ width $\eta$), and is constrained to move in directions that do not decrease its height in the indigo region. Clearly, the shortest path length $d^*(x_0, x_g)$ is a quasipseudometric. We then approximate $d^*$ by $d_\theta$. To test $\theta$'s generalization ability, we consider supervised learning on 20 data points $\{(x_0, x_g), d^*(x_0, x_g)\}$, where $x_0, x_g \in [0,1]^2$ are sampled uniformly at random.\footnote{For simplicity,  $d_\theta$ directly regresses to $d^*$ without doing RL, hence no actions are involved.}
In the rest of Fig.~\ref{fig:toy}, we plot the \emph{generalization} error over training iteration for different thresholds $\eta$, when $\theta$ is either MRN or MRN with only $d_\text{sym}$.\footnote{MRN with only $d_\text{sym}$ is expanded to match the size of MRN.} For reference, we also plot the best generalization error of the bilinear value network (BVN)~\citep{hong2022bilinear} and monolithic network, as other designs including the deep/wide norms~\citep{pitis2020inductive} fail to work in this toy example.
The generalization error is evaluated on 10000 randomly sampled $(x_0, x_g)$ pairs different from the training data. Each curve is plotted to the iteration that reaches the lowest generalization error. From the figures, we observe that \textbf{1)} as threshold $\eta$ increases, the generalization error of MRN decreases because $d^*$ becomes more symmetric; and \textbf{2)} the symmetric part alone (\textcolor{mygreen}{green}) does not approximate $d^*$ well as $d^*$ is asymmetric.

\begin{figure}[t!]
    \centering
    \includegraphics[width=\columnwidth]{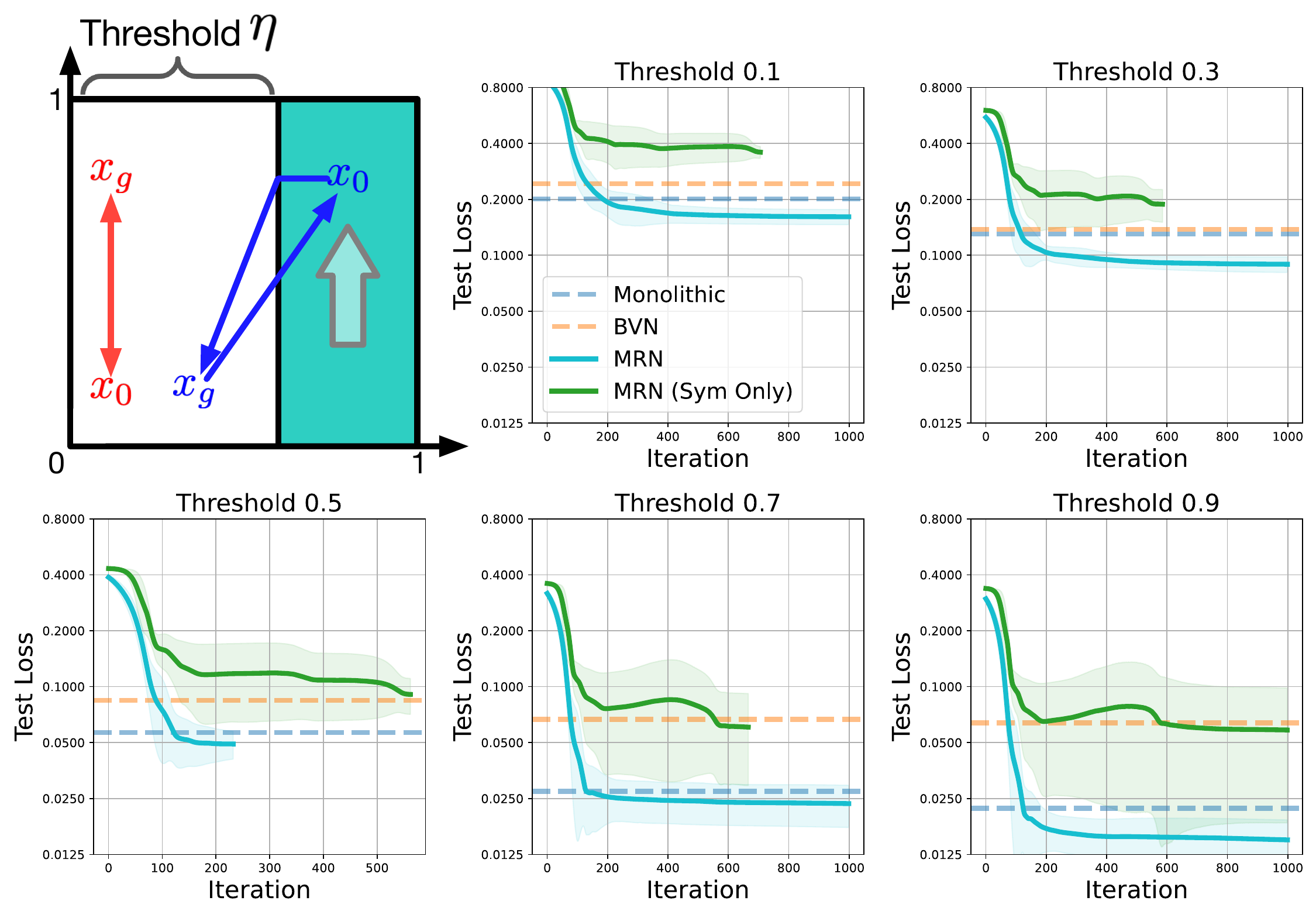}
    \vspace{-20pt}
    \caption{A toy example for illustrating the importance of both the symmetric and asymmetric parts in MRN.}
    \label{fig:toy}
    \vspace{-15pt}
\end{figure}

To summarize, the design of MRN not only incorporates the inductive bias that $Q^*$ must satisfy the triangle inequality but also ensures it can universally approximate any $Q^*$ with good generalization ability by having both the $d_\text{sym}$ and $d_\text{asym}$ parts. For the convenience of understanding and implementation of MRN, we provide the forward pass of MRN in PyTorch-like pseudocode in Alg.~\ref{alg:MRN} in the Appendix.
\mysubsubsection{A Unified View of Existing GCRL Architectures} Recent work proposed Deep Norm (DN) and Wide Norm (WN)~\citep{pitis2020inductive}, which also satisfy the triangle inequality by design. However, both DN and WN are norm-based networks, i.e., they are restricted to only approximate norm-induced quasi-metrics. On the other hand, the conventional monolithic multi-layer perceptron (MLP) modeling $Q(s, a, g)$, and the recently proposed bilinear value network (BVN)~\citep{hong2022bilinear} both inherit the universal approximation guarantee from general neural networks, but they do not enforce the triangle inequality, thus making them less computationally efficient to learn. The Poisson Quasi-metric Embedding (PQE)~\citep{wang2022learning} enjoys the same theoretical guarantee as MRN. However, MRN is simpler in structure and empirically performs better than PQE (See Sec.~\ref{sec:exp-main}).
As such, all existing designs discussed so far can be summarized in Fig.~\ref{fig:venn}.
\begin{figure}[h!]
    \centering
    \includegraphics[width=0.9\columnwidth]{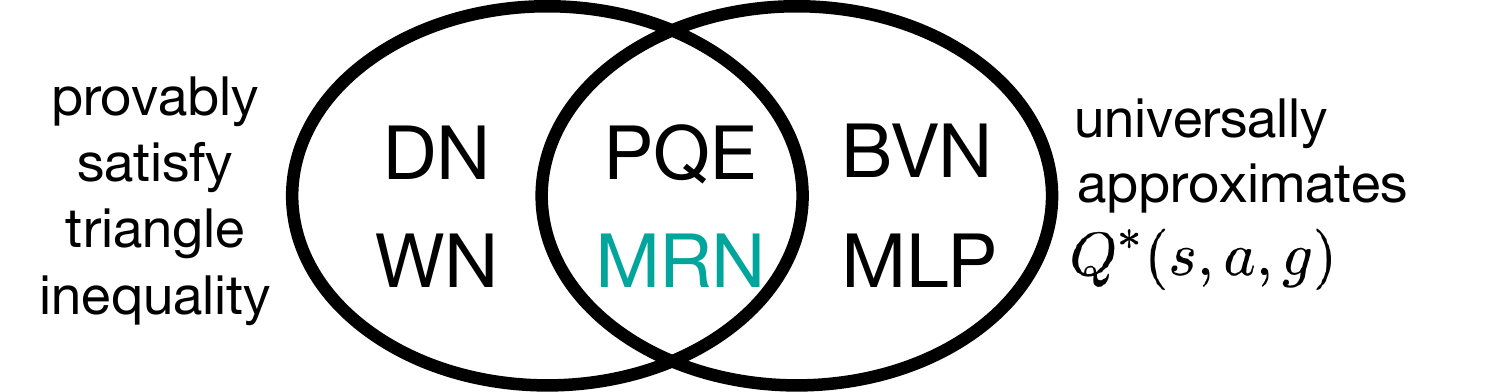}
    \caption{A Venn diagram for existing networks for GCRL.}
    \vspace{-15pt}
    \label{fig:venn}
\end{figure}
\mysection{Related Work}
A brief summary of previous attempts to approximate the action-value function is presented in the following.

Early work considered linear functions as parametric models for value function approximation~\citep{sutton2018reinforcement,littman2001predictive,singh2003learning}.  Successor featured~\citep{dayan1993improving,kulkarni2016deep,borsa2018universal} is another line of work that assumes the reward is a linear function of the state-action features (e.g. $R(s,a) = w^\top \phi(s,a)$). 
Therefore, $Q(s,a) = w^\top \psi(s,a)$ where $\psi(s,a)$ is the discounted expected state-action feature occupancy. While linear functions are well-studied and relatively supportive of proving convergence properties, they can be too restrictive to approximate the value function in practical problems. 
Laplacian reinforcement learning~\citep{mahadevan2005value} instead represents the value function based on a set of Fourier series. 
All of the above work discussed so far learns the action-value function in a single-task setting, without considering the more general goal-conditioned setting. 

More recently, researchers have considered how to decompose the \emph{universal value function approximator} $Q(s,a,c)$ where $c$ can be any context variable (e.g. a goal $g$ or a specific task $k$). In particular, for GCRL, \citet{schaul2015universal} consider the low-rank bilinear decomposition, i.e. $Q(s, a, g) = f(s, a)^\top\phi(g)$. The motivation behind bilinear value networks is that decomposing $Q(s,a,g) = f(s,a)^\top \phi(s,g)$ may result in better learning efficiency compared to the low-rank bilinear decomposition~\citep{hong2022bilinear}. 
Besides the above approaches designed specifically for GCRL, \citet{pitis2020inductive} proposed the Deep Norm (DN) and Wide Norm (WN) families of neural networks that respect the triangle inequality. 
Note that norms always respect the triangle inequality, and DN is a network that essentially computes some norm between two points $x$ and $y$. However, to relax one additional assumption of norms called homogeneity, i.e. $D(c\cdot(x-y)) = c D(x-y)$, DN adds a concave function on top of the convex neural network. But even with this relaxation, DN still cannot represent all functions that respect the triangle inequality as the input to DN is restricted to be $x-y$. WN is essentially any linear or maxout combination of DN networks. While DN and WN can approximate a rich family of functions, they are still restricted to norm-induced functions. Concurrently, Possion Quasi-metric Embedding (PQE) has been proposed, which approximates any quasipseudometric from a distribution perspective~\citep{wang2022learning}. PQE improves upon DN/WN as it can universally approximate any quasipseudometric function. However, as explained in detail in  Sec.~\ref{sec:method-triangle-ineq}, when applied to RL problems, PQE considers the more restrictive first hitting-time problem, while we consider the general GCRL setting. 
By explicitly decomposing the network into a metric part plus an asymmetric residual part, MRN can learn more efficiently. In addition, the design of DN, WN, and PQE are relatively complicated and may require careful hyperparameter tuning in practice. By contrast, MRN is much simpler and involves no hyperparameters.

\section{Experimental Results}
\vspace{5pt}
Experiments are designed to validate two hypotheses: 1) MRN achieves better sample efficiency compared to the baseline methods (Sec.~\ref{sec:exp-main}), and 2) $d_\text{sym}$ and $d_\text{asym}$ are both important in the design of MRN (Sec.~\ref{sec:exp-ablation}). We start by introducing the benchmarks and baseline methods and provide the experiment details. Then we provide the results and analysis indicating that they confirm our hypotheses.
\myparagraph{Benchmarks} We use the standard GCRL benchmarks~\citep{plappert2018multi} including all manipulation tasks on the \emph{Fetch} robot and \emph{Shadow-hand} (See Fig.~\ref{fig:env-vis}).
\vspace{-5pt}
\begin{figure}[h!]
    \centering
    \includegraphics[width=\columnwidth]{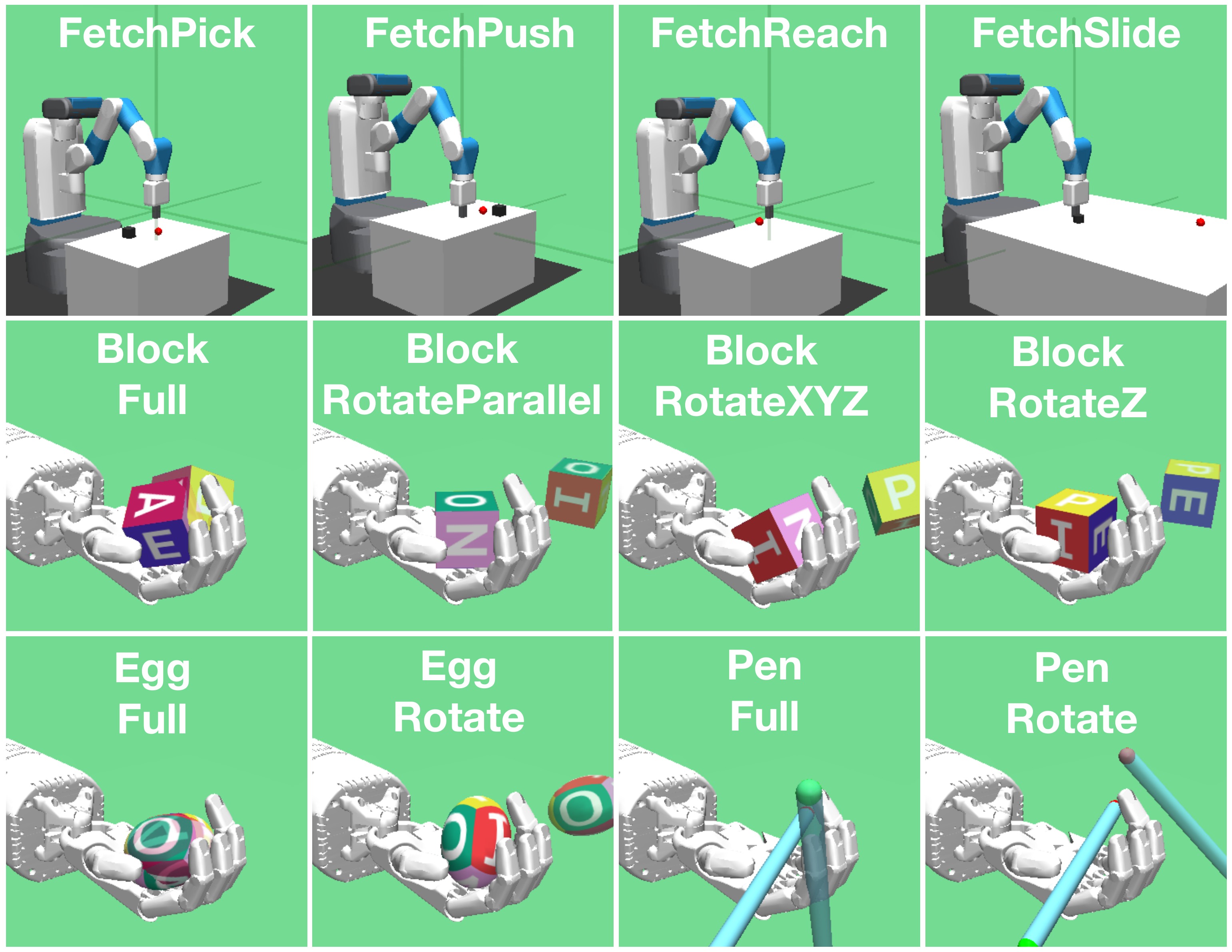}
    \vspace{-20pt}
    \caption{The 12 GCRL environments.}
    \label{fig:env-vis}
    \vspace{-10pt}
\end{figure}
\myparagraph{Baselines} We compare MRN with a comprehensive list of existing UVFA architectures as illustrated in Fig.~\ref{fig:architecture}:
\begin{itemize}
    \item \textit{Monolithic}: $Q(s,a,g)$ is an unconstrained neural network that directly maps $(s,a,g)$ to the value.
    \item \textit{Bilinear Value Network} (BVN)~\citep{hong2022bilinear}: $Q(s,a,g) = f(s,a)^\top \phi(s,g)$.
    \item \textit{Deep/Wide Norms} (DN/WN)~\citep{pitis2020inductive}. DN/WN are networks that represent norm-induced metrics. The original DN/WN operates on the space where the metric is defined. To adapt them to the GCRL setting, we apply DN/WN to $h_{sa}$ and $h_{sg}$ with the same encoders as MRN.
    \item \textit{Poisson Quasi-metric Embedding} (PQE)~\citep{wang2022learning}. PQE represents quasi-metrics from a distribution perspective.
\end{itemize}
\vspace{-4pt}
\myparagraph{Algorithm and architecture} We use deep deterministic policy gradient (DDPG)~\citep{lillicrap2015continuous} with hindsight experience replay (HER)~\citep{andrychowicz2017hindsight} as the base GCRL algorithm, within which we test all different architectures. Specifically, we only change the critic architecture $Q(s,a,g)$. We constrain different networks to have the same number of parameters to equalize learning capacity. The monolithic network is a three-hidden-layer multi-layer perception (MLP) with 256 neurons per layer with ReLU activation (e.g. [\texttt{linear}-\texttt{relu}]$\times 3$ + \texttt{linear}). BVN has two separate networks $f$ and $\phi$, each of which is a three-layer MLP with 176 neurons per layer. For all other networks, we first define two encoders $e_1$ and $e_2$ (e.g. [\texttt{linear}-\texttt{relu}]$\times 2$). DN, WN, and PQE have method-specific modules on top of the two encoders, which follow the corresponding published implementation. For MRN, both the metric part $d_\text{sym}$ and the residual asymmetric part $d_\text{asym}$ are a single hidden layer neural network with 176 neurons (e.g., \texttt{linear}-\texttt{relu}-\texttt{linear}). The actor network is the same as the monolithic critic network except that the output layer projects to the action space.
\myparagraph{Unified implementation on GPU} To our knowledge, there has not been a consistent codebase that compares all existing neural architectures for the GCRL problem. Moreover, past implementations with HER~\citep{andrychowicz2017hindsight} require heavy CPU computation (e.g., 19 threads per run per environment). To address both issues, we implement all existing implementations in a unified framework and train fully on GPUs, which greatly reduces the amount of training time when extensive CPU computation is not available. The implementation only requires $\sim$2000 mebibytes (MiB) per run per environment. Across all architectures, we use the same learning rate of $0.001$ and the Adam optimizer~\citep{kingma2014adam} for updating both the actor and the critic.
\myparagraph{Evaluation} For each architecture and each environment, we evaluate with 5 independent seeds $\{100, 200, 300, 400, 500\}$. The agent is trained on 1000 episodes of data each epoch. After each training epoch, we evaluate the agent by recording its average performance (success rate) over 100 independent rollouts with randomly sampled goals. We plot the average success rate over learning epochs, averaged over the 5 seeds, along with its standard deviation (shaded region).
\subsection{MRN Achieves Better Sample Efficiency}
\label{sec:exp-main}
The GCRL experiment results using different state-of-the-art neural architectures as the critic is presented in Fig.~\ref{fig:exp-main}. Note that within BVN, DN, WN, and PQE, no single method performs uniformly better than all other methods across all environments. By contrast, MRN (\textcolor{cyan}{cyan}) performs comparably to or better than all baseline methods in all environments.

\begin{figure*}[t!]
    \centering
    \includegraphics[width=\textwidth]{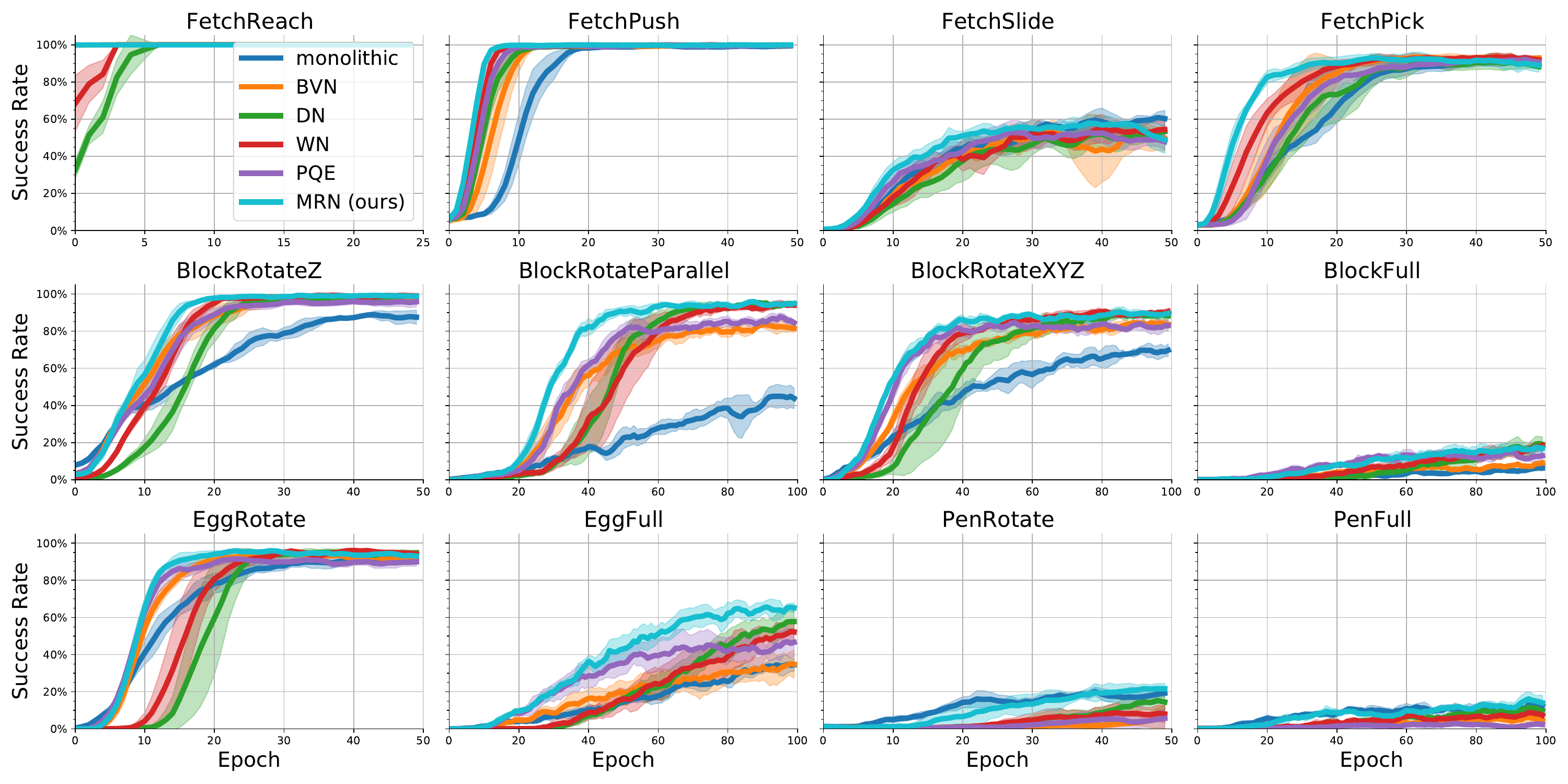}
    \vspace{-20pt}
    \caption{Success rate over training epochs for MRN (\textbf{ours}), the monolithic network, BVN~\citep{hong2022bilinear}, DN and WN~\citep{pitis2020inductive}, and PQE~\citep{wang2022learning}, on 12 GCRL environments from~\citep{plappert2018multi}.}
    \label{fig:exp-main}
    \vspace{-10pt}
\end{figure*}

\begin{figure*}[t!]
    \centering
    \includegraphics[width=\textwidth]{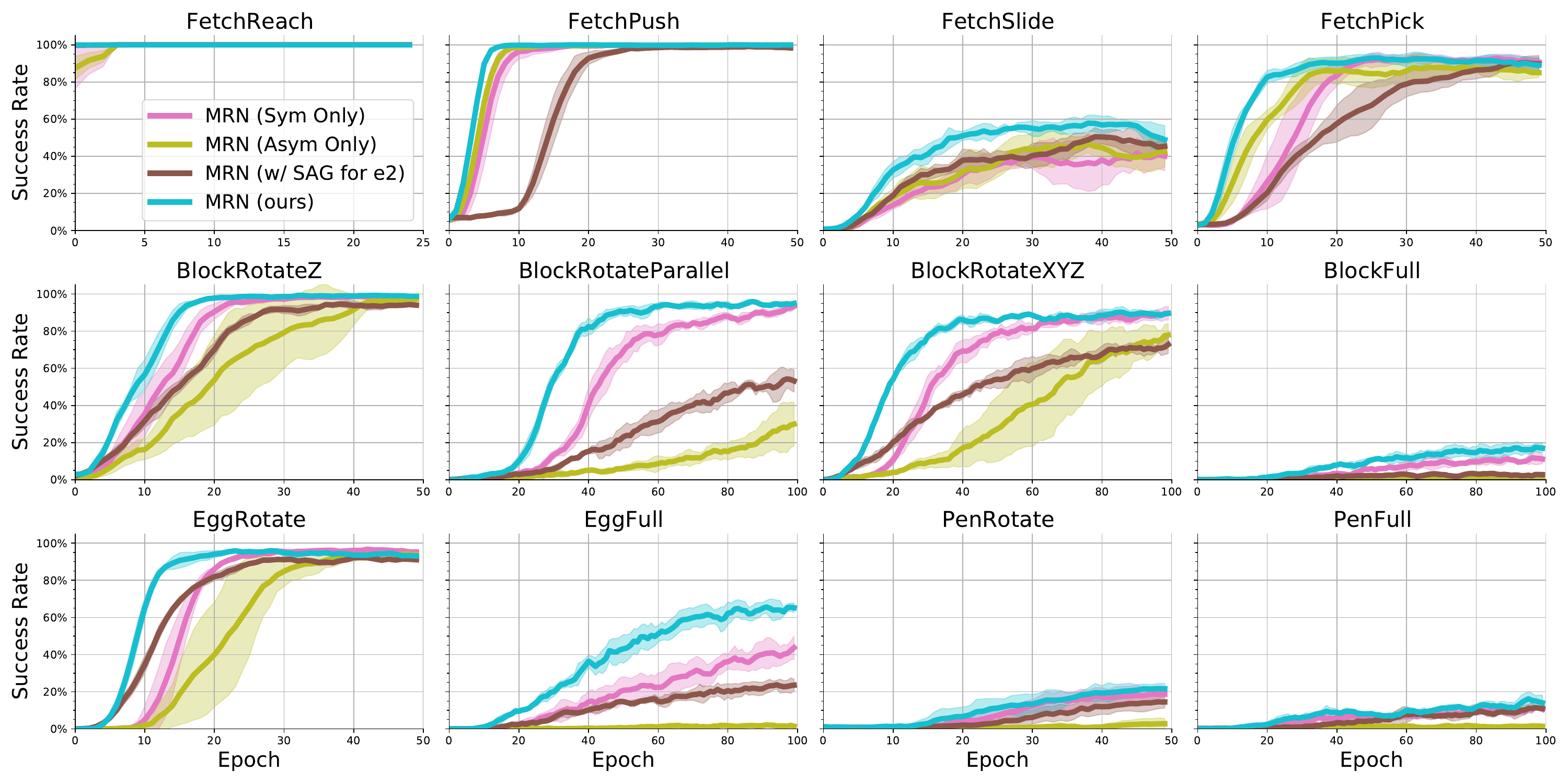}
    \vspace{-20pt}
    \caption{Ablation study on individual symmetric/asymmetric parts of MRN, and on feeding all $(s, a, g)$ to $e_2$.}
    \label{fig:exp-ablation}
    \vspace{-10pt}
\end{figure*}
\subsection{On Importance of Both $d_\text{sym}$ and $d_\text{asym}$}
\label{sec:exp-ablation}
To validate the necessity of decomposing MRN into $d_\text{sym}$ and $d_\text{asym}$, we conduct an ablation study that experiments with each part alone. To match the learning capacity, we enlarge the $d_\text{sym}$-only and $d_\text{asym}$-only networks to have 300 neurons per layer. The results are depicted in Fig.~\ref{fig:exp-ablation}. From the figure, we can see that the $d_\text{sym}$-only network converges faster than the $d_\text{asym}$-only network. However, neither alone achieves the best sample efficiency or final performance. This result further confirms the necessity of combining both parts. To exclude the case that MRN outperforms the ablated networks due to the learning rate, we double the learning rate for the ablated methods, and the result is in Appendix~\ref{sec:apx-exp-abl}. 
Recall from Sec.~\ref{sec:method-mrn} (and as illustrated in Fig.~\ref{fig:architecture}), that whereas in theory, $x_g$ is a function of $s, a, g$, in practice, we only provide $s,g$ as inputs for the sake of efficiency.  To validate this design, we check what happens when $h_{sa} = e_1(s,a)$ but $h_{sg} = e_2(s, a, g)$. The result is in Fig.~\ref{fig:exp-ablation} (w/ SAG for e2). We observe that feeding all $s, a, g$ as input to $e_2$ confuses the network and prevents it from distinguishing between the roles of $e_1$ and $e_2$.
\vspace{1pt}
\mysection{Conclusion}
This paper introduces the metric residual network for GCRL problems that has clear theoretical motivations and is straightforward to implement in practice. Comprehensive studies on standard GCRL benchmarks demonstrate that MRN outperforms all existing designs of neural architectures in terms of sample efficiency. As large models pretrained on a large set of environments become increasingly popular, we expect that investigating useful inductive biases for saving computation will become increasingly interesting and important. An interesting future research direction is to explore effective architectures for more general RL problems where the reward function can take an arbitrary form.

\clearpage
\section{Acknowledgement}

The research was conducted in both the statistical learning and AI group (SLAI) led by Qiang Liu and the Learning Agents
Research Group (LARG) led by Peter Stone in computer science at UT Austin. SLAI research is supported in part by CAREER-1846421, SenSE-2037267, EAGER-2041327, and Office of Navy Research, and NSF AI Institute for Foundations of Machine Learning (IFML).  LARG research is supported in part by NSF
(CPS-1739964, IIS-1724157, FAIN-2019844), ONR (N00014-18-2243), ARO
(W911NF-19-2-0333), DARPA, GM, Bosch, and UT Austin's
Good Systems grand challenge.  Peter Stone serves as the Executive
Director of Sony AI America and receives financial compensation for
this work.  The terms of this arrangement have been reviewed and
approved by the University of Texas at Austin in accordance with its
policy on objectivity in research. Special thanks to Xuchen Ma for the valuable discussion on goal-conditioned reinforcement learning.
\bibliography{aaai23.bib}

\newpage
\appendix
\onecolumn

\section{Theory}

\subsection{Proof of Thm.~\ref{thm-universal}}
\label{sec:apx-thm1}
\begin{thm} Consider the following goal-conditioned MDP $M_{gc} = (\gS, \gA, \gG, T, R, \gamma, \rho_0, \rho_g)$ as described in Sec.~\ref{sec:gcrl} except that $\gG\equiv \gS \times \gA$ . Then the optimal universal value function $Q^*$ defined in Eq.~\eqref{eq:specific-v} satisfies the triangle inequality: $~\forall~ x^1, x^2, x^3 \in \gX$,
    \begin{equation}
       Q^*(x^1, x^2) + Q^*(x^2, x^3) \leq Q^*(x^1, x^3).
    \end{equation}
\end{thm}
\begin{proof}
Assume the statement does not hold. As we fix the goals to be $x^2$ or $x^3$, the GCRL problem is transformed into two standard MDPs with fixed goals. Therefore, there exists an optimal Markov policy for each problem. Specifically, denote optimal Markov policies corresponding to $Q^*(x^1, x^2), Q^*(x^2, x^3)$ and $Q^*(x^1, x^3)$ as $\pi_1$, $\pi_2$ and $\pi_3$. Then consider the following policy $\pi_{1\rightarrow2}$ for $t > 0$:
\begin{equation}
    \pi_{1\rightarrow2}(a \mid s_t) = \begin{cases}
        \pi_1(a \mid s_t) & x^2 \notin x_{< t}, \\
        \pi_2(a \mid s_t) & \text{otherwise}.
    \end{cases}
\end{equation}
Here, $x_{< t} = \{x_0, x_1, \dots, x_{t-1}\}$ denotes all past state and action pairs before time $t$. Let $\tau$ be the random variable that indicates the first time $\pi_{1\rightarrow 2}$ reaches $x^2$. Then we define
$$
q^1_{1\rightarrow 2} = \mathbb{E}_{\substack{(x_t, r_t) \\\sim \pi_{1\rightarrow2}, T, R, \tau}} \bigg[ \sum_{t=0}^\tau \gamma^t r_{t,g} ~\big|~x_0 = x^1, g = x^2\bigg],
$$
and
$$
q^2_{2\rightarrow 3} = \mathbb{E}_{\substack{(x_t, r_t) \\\sim \pi_{1\rightarrow2}, T, R, \tau}} \bigg[ \sum_{t=\tau}^\infty \gamma^t r_{t,g} ~\big|~x_\tau=x^2, g = x^3\bigg].
$$
Then, clearly
\begin{equation}
\begin{split}
Q^{\pi_{1\rightarrow 2}}(x^1, x^3) &= q^1_{1\rightarrow 2} + q^2_{1\rightarrow 2} \\
&\geq Q^*(x^1, x^2) + Q^*(x^2, x^3) > Q^*(x^1, x^3).
\end{split}
\label{eq:thm1-chain}
\end{equation}
The first equality follows the definition of $Q^{\pi_{1\rightarrow 2}}$. To see why the first inequality holds, note that
\begin{align*}
&Q^*(x^1, x^2) - q^1_{1\rightarrow 2} \\
= &\mathbb{E}_{\substack{(x_t, r_t) \\\sim \pi_{1\rightarrow2}, T, R, \tau}} \bigg[ \sum_{t=\tau}^{\infty} \gamma^t r_{t,g} ~\big|~x_\tau =x^2, g = x^3\bigg] \\
\leq &0~~~~\text{(since $r_{t,g} \leq 0$)}.
\end{align*}
Similarly, we have
$$
     Q^*(x^2, x^3) - q^2_{1\rightarrow 2} \leq \mathbb{E}_\tau\big[\gamma^\tau Q^*(x^2, x^3)\big] - q^2_{1\rightarrow 2} \leq 0.
$$
Combining above we can tell that the first inequality holds. The last inequality in Eq.~\eqref{eq:thm1-chain} holds because of our assumption. Therefore, we have a policy $\pi_{1\rightarrow 2}$ (though possibly non-Markov) that achieves a higher expected discounted return with respect to reaching $x^3$ than $\pi^3$, which contradicts the assumption that $\pi^3$ is optimal.
\end{proof}

\subsection{Proof of Prop.~\ref{prop:quasipseudometric}}
\label{sec:apx-prop-quasipseudometric}
\begin{pro}
    Assume $x, y \in \gX$, define
\begin{equation}
    d(x, y) = d_\text{sym}(x, y) + \underbrace{\max_{i \in [K]} \big(h_i(x) - h_i(y)\big)_{+}}_{d_\text{asym}(x, y)},
    \label{eq:qm2}
\end{equation}
where $(\cdot)_+ = \max(\cdot, 0)$ is the Rectified Linear Unit (ReLU) function, $h_i: \gX \rightarrow \mathbb{R}$ and $(\gX, d_\text{sym})$ is a metric. Then $d$ ensures the three axioms of quasipseudometrics:
\begin{itemize}
    \item Non-negativity:
    $$
        \forall~ x, y \in \gX, ~~~~ d'(x,y) \geq 0.
    $$
    \item Identity:
    $$
        \forall~x \in \gX, ~~~~ d'(x, x) = 0.
    $$
    \item Triangle inequality:
    $$
        \forall~ x, y, z \in \gX, ~~~~d'(x,z) \leq d'(x, y) + d'(y,z).
    $$
\end{itemize}
\end{pro}
\begin{proof}
The non-negativity is obvious because $d$ is a metric and we use the ReLU function. The identity property can be easily checked. For the triangle inequality,
\begin{flalign*}
    &d'(x, y) + d'(y, z) = d(x, y) + d(y, z) \\
    & \quad + \max_i \big(h_i(x) - h_i(y)\big)_+ + \max_i \big(h_i(y) - h_i(z)\big)_+ \\
    &\geq d(x, z) + \max_i \big(h_i(x) - h_i(y) + h_i(y) - h_i(z)\big)_+ \\
    &= d'(x, z).
\end{flalign*}
\end{proof}

\subsection{Proof of $ Q^*(x, g) = \sup_{x': M(x') = g} Q^*(x, x')$ in Eq.~\eqref{eq:qxg_is_qxx}}
\label{sec:apx-obs}
Assume the $\sup$ is attainable. Fixing a goal $g$, the GCRL problem becomes the standard MDP and there exists a deterministic optimal policy $\pi^*$ for reaching $g$ from $x$. Since we assume the dynamics is also deterministic, then it means the optimal trajectory to reach $g$ from $x$ is also deterministic. Denote this trajectory as $\xi = (s_0, a_0, \dots)$ where $(s_0, a_0) = x$. Also denote
\begin{equation}
    x_g = \argmax\limits_{x': M(x') = g} Q^*(x, x')
\end{equation}
Then we argue that $Q^*(x, g) = Q^*(x, x_g)$. Assume otherwise, then either $Q^*(x, x_g) > Q^*(x, g)$ or $Q^*(x, x_g) < Q^*(x, g)$.
\begin{itemize}
    \item $Q^*(x, x_g) > Q^*(x, g)$: \\
    this is impossible because otherwise we can replace $\pi^*$ with the policy corresponding to $Q^*(x, x_g)$ to reach a higher return, contradicting the definition of $\pi^*$.
    \item $Q^*(x, x_g) < Q^*(x, g)$: \\
    Denote 
    $$
        \tau = \min_t (M(x_t) = g),~~~ x_t \in \xi.
    $$
    In other words, $x_\tau$ is the first $(s, a)$ pair along $\xi$ that reaches the goal. Then there are two cases:
    \begin{itemize}
        \item After $\pi^*$ reaches $x_\tau$, it will come back to $x_\tau$ again (hence repeatedly). In this case, clearly $$Q^*(x, x_g) \geq Q^*(x, x_\tau) \geq Q^*(x, g),$$ contradicting our assumption that $Q^*(x, g) > Q^*(x, x_g)$. The first inequality follows the definition of $x_g$. The second inequality follows the definition of $Q^*(x, x_\tau)$.
        \item After $\pi^*$ reaches $x_\tau$, it will \emph{never} come back to $x_\tau$ again. In this case, one can find the next $\tau' = \min_{t > \tau} (M(x_t) = g),~~x_t \in \xi$. Then similarly there are two cases, if $\pi^*$ repeatedly visits $x_{\tau'}$, then we follow the argument in the above case. Otherwise we can recursively find the next $\tau''$, so on and so forth. Eventually, either we end up with the last state $x_\zeta$ such that no $t > \zeta$ satisfies $M(x_t) = g$, from which then we can see that 
        $$
        Q^*(x, x_g) \geq Q^*(x, x_\zeta) \geq Q^*(x, g).
        $$
        The last inequality follows that ever since we visit $x_\zeta$, all future rewards will be $-1$. Or the other case is that there exists an infinite length sequence of such $\{x_\tau\}$. Then following this sequence, the original statement is still true though the $\sup$ is no longer attainable. But in this case, one can find a $x_\tau$ in the sequence such that $Q^*(x, x_\tau)$ is arbitrarily close to $Q^*(x, g)$.
    \end{itemize}
\end{itemize}

\subsection{Proof of Thm.~\ref{thm-universal}}
\label{sec:apx-thm2}
\begin{thm}[Universal approximation of $d$]
Given any continuous function $\nu:\gX \times \gX \rightarrow \mathbb{R}$, where $(\gX, \nu)$ is a quasipseudometric and $\gX$ is compact. Then $\forall \epsilon > 0$, with a sufficiently large $K$, there exists a quasipseudometric $d$ in the form of Eq.~\eqref{eq:qm} such that
$$
\forall x, y \in \gX, ~~|d(x, y) - \nu(x, y)| \leq \epsilon.
$$
\vspace{-12pt}
\end{thm}
\begin{proof}
First of all, from Prop.~\ref{prop:quasipseudometric}, any $d$ in the form of Eq.~\eqref{eq:qm} is a quasipseudometric. Next we show that any quasipseudometric $\nu$ can be approximated by a $d$ in the form of Eq.~\eqref{eq:qm}. Since $\nu$ is a quasipseudometric, we have
 \begin{equation}
     \forall~x, y, z \in \gX,~~~\nu(x, y) \geq \nu(x, z) - \nu(y, z)
     ~~~\Longrightarrow~~~ \nu(x, y) = \sup_{z' \in \gX} \bigg(\nu(x, z') - \nu(y, z')\bigg)_+.
    \label{eq:thm-ua-1}
 \end{equation}
 As $\nu(x_1, x_2)$ is continuous in both $x_1$ and $x_2$, for a fixed pair of $(x, y)$, define
 \begin{equation}
     f(z) = \bigg(\nu(x, z) - \nu(y, z)\bigg)_+.
 \end{equation}
 Then $f$ is clearly continuous in $z$ (continuous functions are closed under subtraction and max operation). By definition of continuity,
 \begin{equation}
 \forall ~\epsilon > 0,~~\exists ~\delta > 0,~~~~~d'(z_1, z_2) \leq \delta \Longrightarrow |f(z_1) - f(z_2)| \leq \epsilon,
 \label{eq:lem-2}
 \end{equation}
 where $d'$ is a metric on $\gX$. Since $\gX$ is compact, for a sufficient large $K$, we can cover the domain $\gX$ with $K$ $\delta$-radius balls centered at $\{\xi_i\}_{i=1}^K$, i.e.
$$
\gX \subseteq \bigcup_{i=1}^K B_i,~~\text{where}~~ B_i = \{x \in \gX \mid d'(\xi_i, x) \leq \delta\}.
$$
Now, consider the function $d$, where for any $x, y \in \gX$,
\begin{equation}
    d(x, y) = \max_i \bigg( \nu(x, \xi_i) - \nu(y, \xi_i) \bigg)_+.
\end{equation}
Denote $z^*(x,y) = \argmax_{z'} (v(x, z') - v(y, z'))_+$, then we have $\forall x, y \in \gX$,
\begin{equation}
    \begin{split}
    |d(x, y) - \nu(x, y)| 
    &= \bigg|\max_i \bigg( \nu(x, \xi_i) - \nu(y, \xi_i) \bigg)_+ - \bigg(\nu(x, z^*(x, y)) - \nu(y, z^*(x, y))\bigg)_+ \bigg| \\
    &= \bigg|\max_i \bigg(f(\xi_i) - f(z^*(x, y)) \bigg)\bigg|\\
    &\leq \max_i \bigg|f(\xi_i) - f(z^*(x, y)) \bigg| \\
    &\leq \epsilon \ant{~$\exists~ \xi_i, d'(\xi_i, z^*(x,y)) \leq \delta$, then apply Eq.~\eqref{eq:lem-2}}
    \end{split}
\end{equation}
Clearly the $d$ defined above is in the form of Eq.~\eqref{eq:qm} (we can set $d_\text{sym}(x,y) = 0, \forall x, y \in \gX$).
\end{proof}

\section{Algorithm}
The forward pass of MRN in PyTorch-like pseudocode is provided in Alg.~\ref{alg:MRN}. Here, $e_1$, $e_2$, $d_\text{sym}$ and $d_\text{aym}$ are standard multi-layer perceptrons (MLPs).
\begin{algorithm}[h!]
\begin{flushleft}
\PyInput{The forward pass in PyTorch code} \\ 
\PyCode{def forward(self, state, action, goal):}\\
\qquad\PyComment{state~~(batch, dim\_state)} \\
\qquad\PyComment{action~(batch, dim\_action)} \\
\qquad\PyComment{goal~~~(batch, dim\_goal)} \\
\PyCode{}\\
\qquad\PyComment{1. Projection} \\
\PyCode{~~~~sa = torch.cat([state, action], -1)}\\
\PyCode{~~~~sg = torch.cat([state, goal], -1)}\\
\PyCode{~~~~x ~= self.e\_1(sa)}\\
\PyCode{~~~~y ~= self.e\_2(sg)}\\
\PyCode{}\\
\qquad\PyComment{2. Enforcing triangle inequality} \\
\PyCode{~~~~d\_sym ~= self.sym(x) - self.sym(y)}\\
\PyCode{~~~~d\_sym ~= d\_sym.pow(2).mean(-1).sqrt()}\\
\PyCode{~~~~d\_asym = self.asym(x) - self.asym(y)}\\
\PyCode{~~~~d\_asym = F.relu(d\_asym).max(-1)[0]}\\
\PyCode{~~~~Q\_sag ~= -(d\_sym + d\_asym).view(-1,1)}\\
\PyCode{~~~~return Q\_sag}\\
\end{flushleft}
\caption{Computation of {\PyCode{Q($s, a, g$)}} with MRN}
\label{alg:MRN}
\end{algorithm}

\section{Experiment} 
\subsection{Ablation study with doubled learning rate}
\label{sec:apx-exp-abl}
We conduct the same ablation as in Sec.~\ref{sec:exp-main} except that we double the learning rate to $0.002$. Results are shown in Fig.~\ref{fig:abl-2}. From the figure, we can see that the claim in the main experiment section holds true.
\begin{figure*}[t!]
    \centering
    \includegraphics[width=\textwidth]{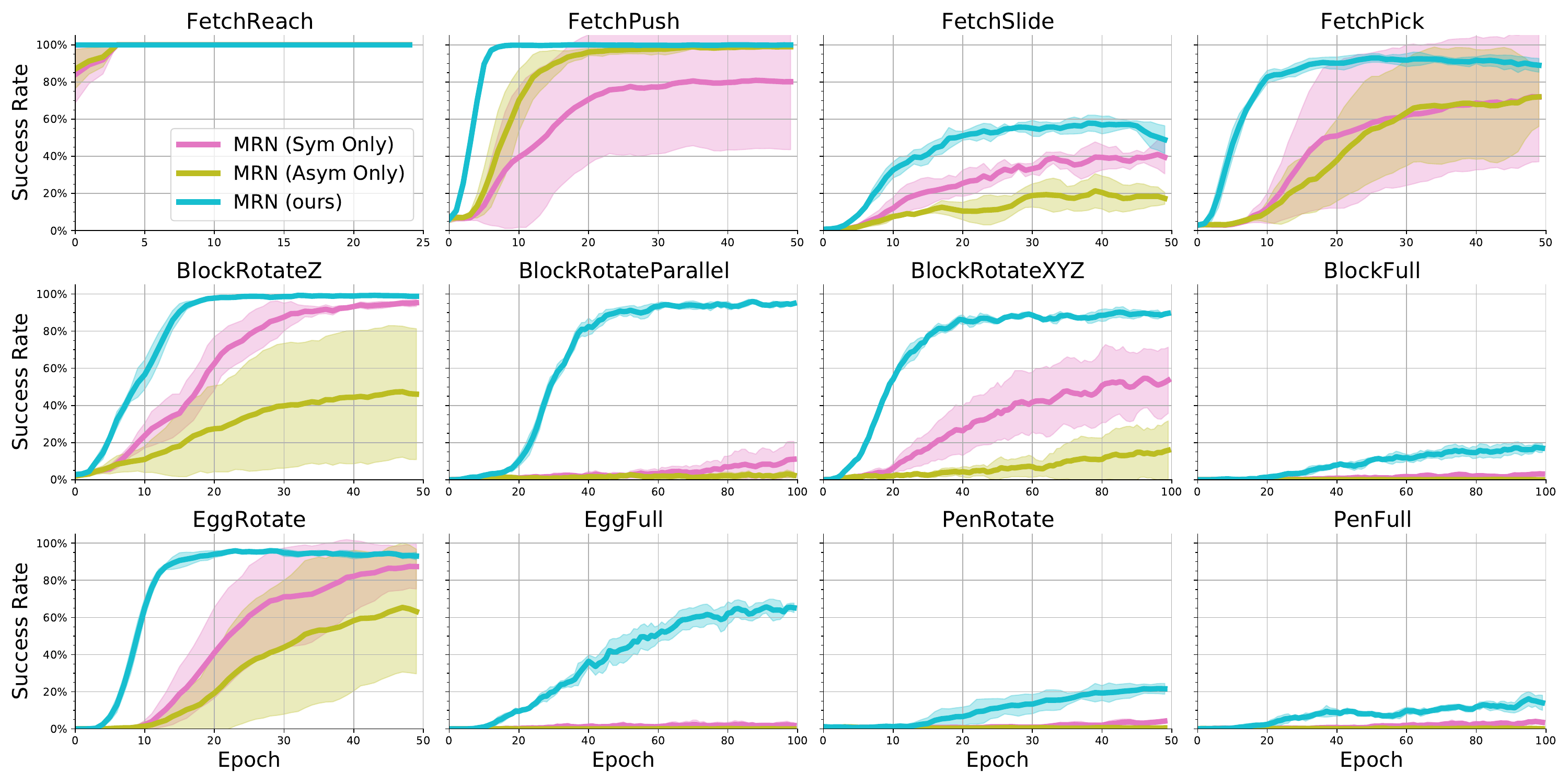}
    \caption{Ablation study on individual symmetric/asymmetric parts of MRN with the doubled learning rate (0.002).}
    \label{fig:abl-2}
\end{figure*}

\end{document}